\documentclass[10pt]{article} 

\usepackage[top=0.9in,bottom=1.1in,left=1.05in,right=1.05in]{geometry}
\title{Reinforcement Learning with Function Approximation: From Linear to Nonlinear}
\usepackage{times}
\usepackage{graphicx}
\usepackage{xr-hyper}
\usepackage{upgreek}
\usepackage{amssymb}
\usepackage{amsfonts}
\usepackage{mathrsfs}
\usepackage{fontenc}
\usepackage{empheq}
\usepackage{enumerate}
\usepackage{mathtools,xpatch}
\usepackage[shortlabels]{enumitem}
\usepackage{algorithm}
\usepackage{amsthm}
\usepackage[ruled,algo2e]{algorithm2e}
\usepackage{bm}
\mathtoolsset{showonlyrefs=true}
\DeclareMathOperator*{\argmin}{arg\,min} 
\DeclareMathOperator*{\argmax}{arg\,max}
\let\inf\relax \DeclareMathOperator*\inf{\vphantom{p}inf}

\newtheorem{defn}{Definition}
\newtheorem{thm}{Theorem}

\newcommand{\bR}{\mathbb {R}}
\newcommand{\bH}{H}
\newcommand{\bP}{\mathbb{P}}
\newcommand{\bS}{\mathcal{S}}
\newcommand{\bA}{\mathcal{A}}
\newcommand{\EE}{\mathbb{E}}

\newcommand{\cond}{\,|\,}
\newcommand{\transpose}{^{\operatorname{T}}}
\newcommand{\rmd}{\,\mathrm{d}}

\date{\today}
\author{
 Jihao Long\thanks{Princeton University}
 \and Jiequn Han\thanks{Flatiron Institute}
}
\makeatletter
 \let\Ginclude@graphics\@org@Ginclude@graphics
\makeatother

\begin{document}

\maketitle

\begin{abstract}
Function approximation has been an indispensable component in modern reinforcement learning algorithms designed to tackle problems with large state spaces in high dimensions. This paper reviews recent results on error analysis for these reinforcement learning algorithms in linear or nonlinear approximation settings, emphasizing approximation error and estimation error/sample complexity. We discuss various properties related to approximation error and present concrete conditions on transition probability and reward function under which these properties hold true. Sample complexity analysis in reinforcement learning is more complicated than in supervised learning, primarily due to the distribution mismatch phenomenon. With assumptions on the linear structure of the problem, numerous algorithms in the literature achieve polynomial sample complexity with respect to the number of features, episode length, and accuracy, although the minimax rate has not been achieved yet. These results rely on the $L^\infty$ and UCB estimation of estimation error, which can handle the distribution mismatch phenomenon. The problem and analysis become substantially more challenging in the setting of nonlinear function approximation, as both $L^\infty$ and UCB estimation are inadequate for bounding the error with a favorable rate in high dimensions. We discuss additional assumptions necessary to address the distribution mismatch and derive meaningful results for nonlinear RL problems.
\end{abstract}

\begin{keywords}%
  Reinforcement learning, Function approximation, High-dimensionality analysis, Distribution mismatch
\end{keywords}

\section{Introduction}
\label{sec_intro}
Reinforcement learning (RL) studies how an agent can learn, through interaction with the environment, an optimal policy that maximizes his/her long-term reward~\cite{sutton2018reinforcement}. When the problem involves a finite set of states and actions of moderate size, the corresponding value or policy functions can be represented precisely as a table, which is called the tabular setting. 
However, when the problem contains an enormous number of states or continuous states, often in high dimensions, function approximation must be introduced to approximate the involved value or policy functions. With the rapid development of machine learning techniques for function approximation, modern reinforcement learning (RL) algorithms increasingly rely on function approximation tools to tackle problems with growing complexity, including video games \cite{mnih2013playing}, Go \cite{silver2016mastering}, and robotics \cite{kober2013reinforcement}.

Despite the astonishing practical success of RL with function approximation when applied to challenging high-dimensional problems, the theoretical understanding of RL algorithms with function approximation remains relatively limited, particularly when compared to the theoretical results in the tabular setting.
In the tabular setting, roughly speaking, we can achieve the minimax sample complexity up to the logarithm term: we need samples of the order of $\frac{H^3|\bS||\bA|}{\epsilon^2}$ to obtain an $\epsilon$-optimal policy, where $H$ denotes the episode length, $|\bS|$ and $|\bA|$ denote the size of the state space and action space (see \cite{gheshlaghi2013minimax,azar2017minimax,dann2017unifying} for detailed discussions). Apparently, these kinds of results become vacuous when $|\bS|$ (and/or $|\bA|$) is extremely large or infinite. Therefore, the study of sample complexity in the presence of function approximation has received considerable attention in recent years in the RL community. Relatively simple function approximation methods, such as the linear model in \cite{yang2019sample,jin2020provably} or generalized linear model in \cite{wang2019optimism,li2021sample} have been examined in the context of RL algorithms. Meanwhile, nonlinear forms like kernel approximation~\cite{domingues2020regret,yang2020provably,yang2020function,long20212,long2022perturbational} have also been studied in RL problems to further bridge the gap between theoretical results under restrictive assumptions and practice. 

In this paper, we review recent theoretical results in RL with function approximation, from linear setting to nonlinear setting. 
We mainly focus on the results regarding approximation error and estimation error/sample complexity, which are errors introduced by function approximation and finite datasets, respectively.
We first review the basic concepts of RL in Section \ref{sec: MDP} and introduce two categories of RL algorithms: value-based methods and policy-based methods in Section \ref{sec:alg}. We are interested in these algorithms when combined with function approximation. In Section \ref{Sec:general_frame}, we give a general framework for the theoretical analysis of RL with function approximation. We adopt the concepts of approximation error, estimation error, and optimization error from supervised learning to RL and discuss the crucial challenges of analyzing these errors in RL. In Section \ref{sec:linear}, we introduce RL algorithms with linear function approximation, as it is the simplest function approximation. We introduce the basic linear MDP assumption~\cite{jin2020provably}, which assumes that both reward function and transition probability are linear with respect to $d$ known features. Under this or similar assumptions, the Q-value function can be represented as a linear function with respect to the features, and numerous algorithms in the literature can achieve polynomial sample complexity with respect to the number of features $d$, episode length $H$, and accuracy $\epsilon$. However, the minimax sample complexity has not been achieved yet.

In Section \ref{sec:nonlinear}, we further discuss RL with nonlinear function approximation. We first introduce the theoretical results of supervised learning on reproducing kernel Hilbert space, neural tangent kernel, and Barron space, and then discuss how to analyze the approximation error in RL problems with nonlinear function approximation. We then focus on the distribution mismatch phenomenon, which is a crucial challenge of RL compared to supervised learning when analyzing the estimation error in the presence of function approximation. In tabular and linear settings, the distribution mismatch is handled by the $L^\infty$ and UCB estimation. However, as we will point out, both $L^\infty$ and UCB estimation suffer from the curse of dimensionality for various function spaces, including neural tangent kernel, Barron space, and many common reproducing kernel Hilbert spaces. This challenge reveals an essential difficulty of RL problems with nonlinear function approximation, and thus additional assumptions are needed to derive meaningful results for nonlinear RL in the literature, including assumptions on the fast eigenvalue decay of the kernel and assumptions on the finite concentration coefficient. We finally introduce the perturbational complexity by distribution mismatch in \cite{long2022perturbational}, which quantifies the difficulty of a large class of the RL problems in the nonlinear setting, as it can give both lower bound and upper bound of the sample complexity of these RL problems. Directions for future work are discussed in Section~\ref{sec:conclusion}.

\paragraph{Notations.} Given $\mathcal{X}$ as an arbitrary subset of Euclidean space, we use $C(\mathcal{X})$ to denote the bounded continuous function space on $\mathcal{X}$, and $\mathcal{P}(\mathcal{X})$ to denote the probability distribution space on $\mathcal{X}$. For any random variable $X$, we use $\mathcal{L}(X)$ to denote its law.
Given a positive integer $H$, $[H]$ denotes the set $\{1,\dots,H\}$. $\mathrm{I}_d$ denotes the identity matrix of size $d$. The notation $\tilde{O}(\cdot)$ ignores poly-logarithmic factors.

\section{Preliminary}\label{sec: MDP}

\subsection{Markov Decision Processes}
Throughout this article, otherwise explicitly stated, we mainly focus on the finite-horizon Markov decision process (MDP) $M = (\bS,\bA,H,P,r,\mu)$ with general time-inhomogeneity as the mathematical model for the RL problem. The specifications are the following.
\begin{itemize}
    \item $\bS$ is the state space and we assume $\bS$ is a subset of a Euclidean space.
    \item $\bA$ is the action space and we assume $\bA$ is a compact subset of a Euclidean space.
    \item $H$ is the length of each episode.
    \item $P\colon [\bH]\times\bS\times\bA \mapsto \mathcal{P}(\bS)$ is the state transition probability. For each $(h,s,a) \in [\bH]\times\bS\times\bA$. $P(\,\cdot\cond h,s,a)$ denotes the transition probability for the next state at step $h$ if the current state is $s$ and action $a$ is taken.
    \item  $r\colon [\bH] \times \bS \times\bA \mapsto \bR$ is the reward function, denoting the  reward at step $h$ if we choose action $a$ at the state $s$. Unless explicitly stated, we assume $r$ is deterministic and the range of $r$ is a subset of $[0,1]$. We also assume that $r$ is a continuous function.
    \item $\mu \in \mathcal{P}(\bS)$ is the initial distribution.
\end{itemize}
We denote a policy by $\pi = \{\pi_h\}_{h=1}^H \in \mathcal{P}(\bA\cond\bS,H)$, where
\begin{align}\label{def_col_distribution}
    \mathcal{P}(\bA\cond\bS,H) = \Big\{\{\pi_h(\,\cdot \cond \cdot\,)\}_{h=1}^H\colon \pi_h(\,\cdot \cond s) \in \mathcal{P}(\bA)
    \text{ for any }s \in \bS \text{ and } h \in [H]\Big\}.
\end{align}

In some cases, we need to work with time-homogeneous, infinite-horizon MDP $M = (\bS,\bA,\gamma, P,r,\mu)$, where
\begin{itemize}
    \item $\bS,\bA$ and $\mu$ are the same with the finite-horizon case,
    \item $\gamma \in [0,1)$ is the discount factor,
    \item $P:\bS \times \bA \mapsto \mathcal{P}(\bS)$ is the state transition probability,
    \item $r:\bS\times \bA \mapsto \bR$ is the deterministic reward function.
\end{itemize}

\subsection{Total Reward, Value Function and Bellman Equation}
Given an MDP $M$ and a policy $\pi$, the agent's total reward is defined according to the following interaction protocol with the MDP.
The agent starts at an initial states $S_0 \sim \mu$; at each time step $h \in [H]$ the agent takes action $A_h \sim \pi_h(\,\cdot \cond S_h)$, obtains the reward  $r(h,S_h,A_h)$ and observes the next state $S_{h+1}\sim P(\,\cdot\cond h, S_h,A_h)$. In this way, we generate a trajectory $(S_0,A_0,\dots,S_H,A_H)$ and we will use $\bP_{M,\pi}$ and $\EE_{M,\pi}$ to denote the probability and expectation of the trajectory generated by policy $\pi$ on the MDP $M$. The expected total reward under policy $\pi$ is defined by
\begin{equation}
    J(M,\pi) = \EE_{M,\pi}[\sum_{h=1}^H r(h,S_h,A_h)],
\end{equation}
and our goal is to find a policy $\pi$ to maximize $J(M,\pi)$ for a fixed MDP $M$. We will use 
\begin{equation}
    J^*(M) = \sup_{\pi \in \mathcal{P}(\bA|\bS,H)} J(M,\pi)
\end{equation}
to denote the optimal value.
For ease of notation in analysis, we will use
$\rho_{h,P,\pi,\mu}$ to denote the distribution of $(S_h,A_h)$ under transition $P$, policy $\pi$ and initial distribution $\mu$.  Moreover, we use $\Pi(h,P,\mu)$ to denote the set of all the possible distributions of $\rho_{h,P,\pi,\mu}$:
\begin{equation*}
    \Pi(h,P,\mu) = \{\rho_{h,P,\pi,\mu}\colon \pi \in \mathcal{P}(\bA \cond \bS,H)\},
\end{equation*}
and let
\begin{equation*}
    \Pi(P,\mu) = \bigcup_{h \in [H]}\Pi(h,P,\mu).
\end{equation*}

Given an MDP $M$, a policy $\pi$, a state $s \in \bS$, an action $a \in \bA$ and time step $h \in [H]$, we define the value function and Q-value function as follows:
\begin{align}
    &V_h^\pi(s) = \EE_{M,\pi}[\sum_{h'=h}^H r(h',S_h',A_h')\cond S_h = s],\\
    &Q_h^\pi(s,a) = \EE_{M,\pi}[\sum_{h'=h}^H r(h',S_h',A_h')\cond S_h' = s, A_h' = a],
\end{align}
as the expected cumulative reward of the MDP starting from step $h$. We have the following Bellman equation:
\begin{align}
    &Q_h^\pi(s,a) = r(h,s,a) + \EE_{s'\sim P(\,\cdot\cond h,s,a)}[V_{h+1}^\pi(s')],\\
    &V_h^\pi(s) = \int_{\bA} Q_h^\pi(s,a)\rmd \pi_h(a|s),
\end{align}
where we define $V_{H+1}^\pi = 0$. The optimal value function and the optimal Q-value function are defined by
\begin{align}
    &V_h^*(s) = \sup_{\pi \in \Pi(\bS,\bA,H)} V_h^\pi(s),\\
    &Q_h^*(s,a) = \sup_{\pi \in \Pi(\bS,\bA,H)} Q_h^\pi(s,a).
\end{align}
The famous Bellman optimality equation gives,
\begin{align}\label{bellman_optimal}
\begin{split}
    & Q_h^*(s,a) = r(h,s,a) + \EE_{s'\sim P(\,\cdot\cond h,s,a)}[V^*_{h+1}(s')],\\
    & V_h^\pi(s) = \max_{a \in \bA}Q_h(s,a). 
\end{split}
\end{align}
where  we again define $V_{H+1}^* = 0$.
We will use $\pi^*$ to denote an optimal policy of $Q_h^*$ satisfying the following greedy condition
\begin{equation}
    \textrm{supp}(\pi_h^*(\,\cdot\cond s)) \subset \{a\in\bA: Q_h^*(s,a) = V_h^*(s)\}
\end{equation}
for any $(h,s,a) \in [H]\times \bS\times \bA$. The optimal policy $\pi^*$ satisfies that $V_h^{\pi^*} = V_h^*$ and $Q_h^{\pi^*} = Q_h^*$. We refer readers to \cite{puterman2014markov} for an in-depth discussion of the value function, Bellman equation, and optimal policy. 

In the case of the time-homogenous MDP, our goal is to find a policy to maximize the discounted total reward
\begin{equation}
    \EE_{M,\pi}[\sum_{h=1}^{\infty}\gamma^{h-1}r(S_h,A_h)].
\end{equation}
We can then introduce the value function, Bellman equation, and optimal policy similarly as in the time-inhomogeneous case, and we refer \cite{puterman2014markov} for details.

\subsection{Simulator Models}
In RL problems, the exact form of the transition probability $P$ and reward function $r$ is unknown, and we can only interact with the MDP to obtain a near-optimal policy. Obviously, the sample complexity of an algorithm depends on the way in which we are allowed to interact with the MDP (the interaction can be different from the interaction according to which the total reward is defined, as described above). There are two main types of simulators for the MDP, which specify the allowed interaction: the generative model setting and the episodic setting. We describe these settings below, and the results in both settings will be reviewed in this paper.

\paragraph{Generative model setting.}  In the generative model setting, one can take any time-state-action tuple $(h,s,a)$ as the input of the simulator and obtain a sample $s' \sim P(\,\cdot\cond h,s,a)$ and the reward $r(h,s,a)$. In this sense, the MDP simulator works as a general generative model.

\paragraph{Episodic setting.} The episodic setting is a more restrictive scenario compared to the generative model setting. In the episodic setting, one can only decide an initial state $s$ and the action on each time step to obtain from the simulator a trajectory $(S_1,A_1,\dots, S_H,A_H)$ and the rewards on the trajectory $r(1,S_1,A_1),\dots, r(H,S_H,A_H)$. In the episodic setting, it is common to consider the \textit{regret}: the cumulative difference between the obtained reward and the optimal reward over $K$ episodes:
\begin{equation}
    \textrm{Regret}(K) = \sum_{k=1}^K[J^*(M) - J(M,\pi_k)]
\end{equation}
for any given policy sequence $(\pi_1,\dots,\pi_K)$ and we aim to minimize the regret over $K$ episodes. If we define a random policy $\bar{\pi}$, which uniformly chooses a policy among $\pi_1,\dots,\pi_K$ and apply it to the MDP, then
\begin{equation}
   J^*(M) - J(M,\bar{\pi}) = \frac{1}{K} \textrm{Regret}(K).
\end{equation}
Therefore, a policy sequence with low regret can generate a near-optimal policy.

\section{RL Algorithms with Function Approximation}\label{sec:alg}
In this section, we introduce some typical RL algorithms with function approximation. They can be divided into two categories:  valued-based methods and policy-based methods.

\subsection{Value-Based Method}\label{value_base_method}
Value-based methods approximate the value or Q-value functions and use the Bellman optimality equation~\eqref{bellman_optimal} to learn the optimal value or Q-value functions. The near-optimal policy can then be obtained through the greedy policy with respect to the optimal Q-value function.

If we have access to a generative model, one typical value-based algorithm with function approximation is the fitted Q-iteration algorithm~\cite{munos2005error,antos2008learning,chen2019information,fan2020theoretical}. Its main idea is as follows: noticing that the conditional expectation minimizes the $L^2$-loss corresponding to the Bellman optimality equation \eqref{bellman_optimal}, we know that for any function space $\mathcal{F}$ such that $Q_h^*\in \mathcal{F}$, $Q_h^*$ is the minimizer of the following optimization problem:
\begin{equation}\label{argmin property}
     \min_{f \in \mathcal{F}}\,\EE_{(s,a)\sim \mu,s'\sim\mathbb{P}_h(\cdot\cond s,a)}|f(s,a) -r(h,s,a) -  V_{h+1}^{*}(s')|^2.
\end{equation}
Therefore, in the fitted Q-iteration algorithm, we compute the $Q_h^*$ backwardly through the Bellman optimality equation \eqref{bellman_optimal}. At each step $h$, we choose $n$ state-action pairs $\{(S_h^i,A_h^i)\}_{1\le i \le n}$, submit the queries $\{(S_h^i,A_h^i,h)\}_{1\le i \le n}$ to the generative model, and obtain the reward and next state $\{(r_h^i,\hat{S}_{h+1}^i)\}_{1\le i \le n}$. 
We solve an empirical version of the least-square problem \eqref{argmin property} for $h = H, H-1,\dots,1$ backwardly. The pseudocode of the fitted Q-iteration is presented in Algorithm \ref{alg_fqi}.
We comment that the performance of the fitted Q-iteration algorithm relies on the samples of state-action pairs $\{(S_h^i,A_h^i)\}_{1\le i\le n}$: we hope these samples are representative enough among those encountered under the optimal policy. 

\begin{algorithm}[ht]
\caption{Fitted Q-iteration algorithm}
\KwIn{MDP ($\bS,\bA,\bH,P,r, \mu$) , function classes $\{\mathcal{F}_h\}_{h=1}^H$, regularization terms $\{\Lambda_h\}_{h=1}^H$, number of samples $n$, state-action pairs $\{(S_h^i,A_h^i)\}_{h \in [H],i \in [n]}$.\\
\textbf{Initialize:} $Q_{H+1}(x,a) = 0$ for any $(x,a) \in \bS \times \bA$. }

\For{$h = H,H-1,\dots,1$}{
Send $(S_h^1, A_h^1, h),\dots,(S_h^n,A_h^n,h)$ to the generative model
and obtain the rewards and next states $(r_h^1,\hat{S}_{h+1}^1), \dots,(r_h^n,\hat{S}_{h+1}^n)$ for all the state-action pairs

Compute $y_h^i = r_h^i + \max_{a \in \bA} Q_{h+1}(\hat{S}_{h+1}^i,a)$

Compute $\hat{Q}_h$ as the minimizer of the optimization problem
\begin{equation}\label{optimization:1}
   \min_{f \in \mathcal{F}_h}\frac{1}{n}\sum_{i=1}^n|y_h^i - f(S_h^i,A_h^i)|^2 + \Lambda_h(f) \;
\end{equation}

Set $Q_h = \max\{0,\min\{\hat{Q}_h,H\}\}$.
}
\KwOut{$\hat{\pi}$ as the greedy policies with respect to $\{Q_h\}_{h=1}^H$.}
\label{alg_fqi}
\end{algorithm}

In contrast to the generative model setting where we can directly choose arbitrary state-action pairs to query at each step, when we work in the episodic setting, we need to decide how to take action on each step in order to balance the trade-off between exploitation and exploration. Exploitation concerns taking actions with large estimated Q-values in order to obtain large rewards, while exploration concerns taking actions with high uncertainty in their Q-values in order to obtain more accurate estimates. On the one hand, taking actions with large estimated Q-values can lead to large estimated total rewards in this episode, but it can also prevent the agent from discovering better actions. On the other hand, taking actions with high uncertainty can lead to more accurate Q-value estimates, but it may also result in lower total rewards in this episode. Finding the proper balance between exploitation and exploration is an important challenge in RL.
The upper confidence bound (UCB) method is a common approach used in reinforcement learning to balance such a trade-off. In the UCB method, a bonus function is added to the estimated Q-value function to reflect the uncertainty in the Q-value estimates. The action that is chosen is the one with the highest sum of the estimated Q-value and the bonus. Since the estimated Q-value and the bonus reflect exploitation and exploration, respectively, the UCB method is widely used in reinforcement learning algorithms and often achieves good performance. The UCB method allows the algorithm to balance the need for exploitation, in order to obtain large immediate rewards, with the need for exploration, in order to obtain more accurate Q-value estimates. We present the pseudocode of one typical such algorithm, the value iteration algorithm \cite{jin2020provably}, in Algorithm \ref{alg_vi}.

\begin{algorithm}[ht]
\caption{Value iteration with function approximation and bonus term}
\KwIn{MDP ($\bS,\bA,\bH,P,r,\mu$), number of episodes $K$, function classes $\{\mathcal{F}_h\}_{h=1}^H$, regularization terms $\{\Lambda_h\}_{h=1}^H$. \\
\textbf{Initialize:} $Q_h^1(s,a) = 0$ for any $(h,s,a) \in [H]\times \bS \times \bA$, $Q_{H+1}^k(s,a) = 0$ for any $(k,s,a) \in [K]\times \bS \times \bA$. }

\For{$k = 1,\dots, K$}{
Sample $S_1^k$ from the initial state distribution $\mu$.

\For{$h = 1,\dots ,H$}{
Take action $A_h^k  = \argmax_{a \in \bA} Q_h^k(S_h^k,A_h^k)$ and observe  the reward $r_h^k = r(h,S_h^k,A_h^k)$ and next state $S_{h+1}^k \sim P(\,\cdot\cond h,S_h^k,A_h^k)$.
}
\If{$k < K$}{
\For{$h = H,H-1,\dots,1$}{
Compute $\hat{Q}_h^{k+1}$ as the minimizer of the optimization problem
\begin{equation}\label{optimization:2}
   \min_{f \in \mathcal{F}_h}\left\{\frac{1}{n}\sum_{i=1}^k|r_h^i+\max_{a \in \bA} Q_{h+1}^{k+1}(S_{h+1}^i,a) - f(S_h^i,A_h^i)|^2 + \Lambda_h(f)\right\}. \;
\end{equation}

Compute the bonus function $b_h^{k+1}:\bS\times\bA \rightarrow [0,+\infty)$ based on $\{(S_h^i,A_h^i)\}_{i=1}^k$ and $\mathcal{F}_h$. Set $Q_{h+1}^{k+1} = \max\{0,\min\{ \hat{Q}_h^{k+1} + b_h^{k+1},H$\}\}. \;
}
}
}
\KwOut{$\hat{\pi}^k$ as the greedy policies with respect to $\{Q_h^k\}_{h=1}^H$ for $k = 1,\dots, K$.}
\label{alg_vi}
\end{algorithm}

\subsection{Policy-Based Method}\label{policy_based_method}
In policy-based methods, the policy function is approximated and optimized based on cumulative reward using stochastic gradient descent. A key step in this process is calculating the gradient with respect to the cumulative reward.
This relies on the policy gradient theorem~\cite{sutton1999policy}:
\begin{equation}
    \nabla_\theta J(\pi_\theta) = \EE_{M,\pi_\theta}[\sum_{h=1}^H\nabla_\theta \log \pi_\theta(A_h\cond S_h)\sum_{h' = h}^Tr(h',S_{h'},A_{h'})].
\end{equation}
Here we only state the naive method to compute the gradient. Several variants can be used to replace $\sum_{h' = h}^Tr(h',S_{h'},A_{h'})$ to reduce the variance, see \cite{schulman2015high} for a detailed discussion. Algorithm \ref{alg_pg} gives the pseudocode of the policy gradient method. Besides the vanilla policy gradient method, several variants of the policy gradient method have been proposed by adding various regularization terms to the cumulative reward as the target of the optimization, including the natural policy gradient \cite{kakade2001natural}, proximal policy optimization~\cite{schulman2017proximal}, and trust region policy optimization~\cite{schulman2015trust}.

\begin{algorithm}[ht]
\caption{Policy gradient method}
\KwIn{MDP ($\bS,\bA,\bH,P,r,\mu$), parametrization of policy $\pi_\theta$, initialization $\theta_0$, batch size $N$, iteration step $K$, learning rate $\eta$.\\
\textbf{Initialize:} Set $\theta = \theta_0$ }

\For{$k = 1,\dots, K$}{
Sample $N$ i.i.d. states $\{S_1^i\}_{i=1}^N$ from $\mu$ and collect $N$ trajectories $\{(S_1^i,A_1^i,\dots,S_H^i,A_H^i)\}_{i=1}^N$ using policy $\pi_\theta$.

Estimate the gradient
\begin{equation}
    g_k = \frac{1}{N}\sum_{i=1}^N\sum_{h=1}^H\nabla_\theta \log \pi_\theta(A_h^i\cond S_h^i)\sum_{h' = h}^Tr(h',S_{h'}^i,A_{h'}^i).
\end{equation}

Update $\theta \leftarrow \theta + \eta g_k$.
}
\KwOut{$\pi_\theta$.}
\label{alg_pg}
\end{algorithm}
\section{General Framework of Theoretical Analysis on RL with Function Approximation}\label{Sec:general_frame}
In this section, we will discuss how to distinguish and quantify different sources that affect the performance of RL algorithms that use function approximation. To begin, we will provide a brief overview of supervised learning and error decomposition in that context. Then, we will examine how to adapt these concepts for application in the analysis of RL algorithms with function approximation.

In supervised learning, the goal is to estimate the target function $f^*$ based on a finite training set
\begin{equation}
    \mathcal{D} = \{x_i,y_i \}_{i=1}^n,
\end{equation}
where $x_1,\dots,x_n$ are i.i.d. sampled from a fixed distribution $\mu$,
\begin{equation}
y_i = f^*(x_i)+\epsilon_i
\end{equation}
and the noises $\epsilon_1,\dots,\epsilon_n$ are i.i.d. standard normal distribution independent of $x_1,\dots,x_n$. We aim to find an estimator $\hat{f}$ with a small population loss
\begin{equation}
    \mathcal{R}(\hat{f}) =\EE_{x\sim \mu}|f^*(x)-\hat{f}(x)|^2.
\end{equation}

In a standard procedure of supervised learning, one first chooses a hypothesis space or a set of trial functions $\mathcal{H}_m = \{f(x;\theta): \theta \in \Theta_m\}$, where $\theta$ denotes the parameters and $m$ denotes the number of parameters in $\mathcal{H}_m$. Common choices of the hypothesis space include linear functions, kernel functions, and neural networks. The next step is to choose a loss function and formulate an optimization problem. The loss function is typically composed of the empirical loss $\mathcal{R}_n(\theta)$ and a regularization term $\Lambda(\theta)$:
\begin{equation}
    \mathcal{L}_n(\theta) = \mathcal{R}_n(\theta) + \Lambda(\theta) = \frac{1}{n}\sum_{i=1}^{n}|y_i - f(x_i;\theta)|^2 + \Lambda(\theta).
\end{equation}
The last step is to solve the optimization problem that aims to minimize the above loss function. It is usually solved by the gradient descent method, stochastic gradient descent method, or their variants.

Let $f_m$ be the minimizer of the population loss $\mathcal{R}(f)$, the best approximation to $f^*$ in $\mathcal{H}_m$, $f_{m,n}$ be the minimizer of the loss function $\mathcal{L}_n(\theta)$ in the hypothesis space $\mathcal{H}_m$ and $\hat{f}$ be the output of the optimization algorithm. Then the total error between the true target function $f^*$ and the output of the supervised learning algorithm $\hat{f}$ can be decomposed into three parts, where $\|\cdot\|_{L^2(\mu)}$ denotes the $L^2$-norm under the distribution $\mu$:
\begin{equation}
    \| f^* - \hat{f}\|_{L^2(\mu)} \leq \underbrace{\|f^*-f_m\|_{L^2(\mu)}}_\textrm{approximation error} +~~ \underbrace{\|f_m-f_{m,n}\|_{L^2(\mu)}}_\textrm{estimation error}~ + \underbrace{\|f_{m,n} - \hat{f}\|_{L^2(\mu)}}_\textrm{optimization error}.
\end{equation}

The first part is the approximation error $\|f^*-f_m\|_{L^2(\mu)}$, which arises because the hypothesis space $\mathcal{H}_m$ may not be able to represent the true function $f^*$ exactly. The second part is the estimation error $\|f_m-f_{m,n}\|_{L^2(\mu)}$, which  arises because we only have a finite dataset and may not be able to find the best approximation $f_m$. The third part is the optimization error $\|f_{m,n} - \hat{f}\|_{L^2(\mu)}$, which arises because the optimization algorithm may not converge to the true minimizer of the empirical loss. We will then discuss the approximation error, estimation error, and optimization error in the context of RL with function approximation.

\subsection{Approximation Error}\label{sec:appro}
To investigate the approximation error in the context of RL, we aim to understand the requirements for the transition probability and reward function needed to accurately approximate the Q-value function for value-based methods and the policy function for policy-based methods. Note that the value function $V^*_h$ is less significant, as we cannot directly compute the optimal policy based on it. The subsequent theorem demonstrates that when the action space is finite and the optimal Q-value function can be accurately approximated, the optimal policy can also be accurately approximated by the corresponding softmax policy. Consequently, our primary focus is on the conditions that ensure the Q-value function can be effectively approximated.
\begin{thm}\label{thm:approximate_optimal_policy}
   Assume that $\mathcal{A}$ is a finite set. Given any $\beta > 0$ and continuous functions $Q = \{Q_h\}_{h=1}^H: \bS\times\bA \rightarrow \bR$, let
\begin{equation}
    \pi_h^{Q,\beta}(a\cond s) = \frac{\exp(\beta Q_h(s,a))}{\sum_{a' \in \bA} \exp(\beta Q_h(s,a'))}.
\end{equation} 
Then,
\begin{equation}
    0 \le J^*(M) - J(M,\pi^{Q,\beta}) \le \frac{H\log|\bA|}{\beta}+ 2 \beta H\sum_{h=1}^H\EE_{M,\pi^{Q^*,\beta}}\max_{a\in\bA}|Q_h^*(S_h,a) - Q_h(S_h,a)|.
\end{equation}
\end{thm}

In supervised learning, a common approach to investigating the approximation error involves proving that the target function resides in specific function spaces, such as reproducing kernel Hilbert space (RKHS) and Barron space. These spaces are ideal for particular approximation schemes, like kernel function and neural network approximation, due to the direct and inverse approximation theorems present within these function spaces. In other words, any function within the space can be approximated using the selected approximation method at a specific rate of convergence, and any function that can be approximated at a particular rate belongs to that function space. We will provide a comprehensive introduction to RKHS and Barron space in Section \ref{sec:space}. 

In RL, our primary focus is on whether the Q-value function, as a function of state and action, lies within the specific function space. Depending on the specific algorithm we analyze, we need to determine the conditions of $P$ and $r$ under which the MDP satisfies the following properties:

\begin{enumerate}[label=Property \arabic*,leftmargin=2.1\parindent,itemindent=2\parindent]
    \item\label{item:approx_Q1}  The optimal Q-value function $Q_h^*$ lies in the specific function space.
    \item\label{item:approx_Q2}  The Q-value function $Q_h^\pi$ lies in the specific function space for any policy $\pi$. This is often required in the policy iteration algorithm with function approximation (see, e.g., \cite{lagoudakis2003least}).
    \item\label{item:approx_Q3} The Bellman optimal operator
    \begin{equation}
        (\mathcal{T}^*_h f)(s,a) = r(h,s,a)+ \EE_{s'\sim P(\,\cdot\cond h,s,a) }[\max_{a'\in\bA}f(s',a')]
    \end{equation}
    maps the specific function space to itself and $\|\mathcal{T}_h^*f\| \le C[1+\|f\|]$ for a constant $C > 0$. This is often required in the fitted Q-iteration and value iteration algorithm with function approximation (see, e.g., \cite{zanette2020learning,yang2020provably,yang2020function}).
    \item\label{item:approx_Q4}  The Bellman operator
     \begin{equation}
        (\mathcal{T}_h f)(s,a) = r(h,s,a)+ \EE_{s'\sim P(\,\cdot\cond h,s,a) }[f(s')]
    \end{equation}
    maps any bounded function in $C(\bS)$ to the specific function spaces and $\|\mathcal{T}_h f\| \le C[1 + \|f\|_{C(\bS)}]$ for a constant $C > 0$.
\end{enumerate}
In these questions, we assume that the function space is a Banach space with norm $\|\cdot\|$ and a subset of the space of all bounded functions in $C(\bS\times\bA)$. Noticing that for any policy $\pi$, we have that $Q_h^\pi = \mathcal{T}_h V_{h+1}^\pi$ and $V_{h+1}^\pi \in [0,H]$, we know that \ref{item:approx_Q4} implies \ref{item:approx_Q2} and hence implies \ref{item:approx_Q1}. Observing that for any bounded function $f$ in $C(\bS\times\bA)$, $\max_{a' \in \bA} f(s',a')$ is a bounded function in $C(\bS)$, we know that \ref{item:approx_Q4} implies \ref{item:approx_Q3}. Moreover, since $Q_{h}^* = \mathcal{T}_h Q_{h+1}^*$, \ref{item:approx_Q3} implies \ref{item:approx_Q1}. Finally, we remark that \ref{item:approx_Q2} and \ref{item:approx_Q3} cannot be inferred from each other \cite[Proposition~5]{zanette2020learning}.
We introduce \ref{item:approx_Q4} because it is the strongest one among those properties and the conditions which imply \ref{item:approx_Q4} are easy to analyze due to the linearity of $\mathcal{T}_h$.

\paragraph{}\noindent
\begin{proof}(of Theorem \ref{thm:approximate_optimal_policy})
By the definition of the optimal value, we have
\begin{equation}
    J^*(M) - J(M,\pi^{Q,\beta}) \ge 0.
\end{equation}
Noticing that
\begin{equation}
   J^*(M) - J(M,\pi^{Q,\beta}) = J^*(M) - J(M,\pi^{Q^*,\beta}) + J(M,\pi^{Q^*,\beta}) - J(M,\pi^{Q,\beta}) := I_1 + I_2,
\end{equation}
we will estimate $I_1$ and $I_2$ respectively.

Using the classical performance difference lemma \cite{kakade2002approximately}, we have
\begin{align}
    I_1 &= \sum_{h =1}^H\EE_{M,\pi^{Q^*,\beta}}\sum_{a\in\bA}Q_h^*(S_h,a)[\pi_h^*(a\cond S_h) - \pi_h^{Q^*,\beta}(a\cond S_h)] \\
    &=\sum_{h =1}^H\EE_{M,\pi^{Q^*,\beta}}\left[\max_{a\in\bA}Q_h^*(S_h,a)-\frac{\sum_{a\in\bA}Q_h^*(S_h,a)\exp(\beta Q_h^*(S_h,a))}{\sum_{a\in\bA}\exp(\beta Q_h^*(S_h,a))}\right].
\end{align}
We will then prove that for any $q = (q_a)_{a\in\bA}$,
\begin{equation}
    q_m- \frac{\sum_{a \in \bA} q_a \exp(\beta q_a)}{\sum_{a\in\bA}\exp(\beta q_a)} \le \frac{\log|\bA|}{\beta},
\end{equation}
where $q_m = \max_{a\in\bA}q_a$.
We can then conclude that 
\begin{equation}\label{eq_I_1}
    I_1 \le \frac{H\log|\bA|}{\beta}.
\end{equation}
Noticing that
\begin{equation}
      q_m - \frac{\sum_{a \in \bA} q_a \exp(\beta q_a)}{\sum_{a\in\bA}\exp(\beta q_a)}= -\frac{(q_a-q_m)\exp(\beta(q_a-q_m))}{\exp(\beta(q_a-q_m))}.
\end{equation}
Define $\phi:\bR^{|\bA|} \rightarrow \bR$:
\begin{align}
    \phi(x) = \beta^{-1}\log(\sum_{a\in\bA}\exp(\beta x_a)).
\end{align}
Then
\begin{equation}
    q_m - \frac{\sum_{a \in \bA} q_a \exp(\beta q_a)}{\sum_{a\in\bA}\exp(\beta q_a)} = -(q-q_m)\transpose \nabla \phi(q-q_m).
\end{equation}
Noticing that $\phi$ is a convex function, we have \cite{gao2017properties}
\begin{equation}
    -(q-q_m)\transpose \nabla \phi(q-q_m) = [0-(q-q_m)]\transpose \nabla \phi(q-q_m) \le \phi(0) - \phi(q-q_m) = \frac{\log|\bA|}{\beta}- \phi(q-q_m). 
\end{equation}
Noticing that there exists an $a\in\bA$ such that $q_a-q_m = 0$, we have
\begin{equation}
    \phi(q-q_m) \ge 0.
\end{equation}
Therefore,
\begin{equation}
     q_m - \frac{\sum_{a \in \bA} q_a \exp(\beta q_a)}{\sum_{a\in\bA}\exp(\beta q_a)} \le \frac{\log|\bA|}{\beta}.
\end{equation}

For $I_2$, we again use the performance difference lemma to obtain that
\begin{align}
    I_2 &= \sum_{h =1}^H\EE_{M,\pi^{Q^*,\beta}}\sum_{a\in\bA}Q_h^{\pi^{Q^*,\beta}}(S_h,a)[\pi_h^{Q^*,\beta}(a\cond S_h) - \pi_h^{Q,\beta}(a\cond S_h)] \\ 
    &\le H\sum_{h =1}^H\EE_{M,\pi^{Q^*,\beta}}\sum_{a\in\bA}|\pi_h^{Q^*,\beta}(a\cond S_h) - \pi_h^{Q,\beta}(a\cond S_h)| \\
    &=H\sum_{h =1}^H\EE_{M,\pi^{Q^*,\beta}}\sum_{a\in\bA}
    \left|\frac{\exp(\beta Q_h^*(S_h,a))}{\sum_{a' \in \bA}\exp(\beta Q_h^*(S_h,a'))} - \frac{\exp(\beta Q_h(S_h,a))}{\sum_{a' \in \bA}\exp(\beta Q_h(S_h,a'))}\right|.
\end{align}
Given $q = \{q_a\}_{a\in\bA}$ and $\bar{q} = \{\bar{q}_a\}_{a\in \bA}$, we have
\begin{align}
    &\sum_{a\in\bA}
    \left|\frac{\exp(\beta q_a)}{\sum_{a' \in \bA}\exp(\beta q_{a'})} - \frac{\exp(\beta \bar{q}_a)}{\sum_{a' \in \bA}\exp(\beta \bar{q}_{a'})}\right| \\
    \le& \frac{\sum_{a\in\bA}\sum_{a' \in \bA}|\exp(\beta q_a + \beta \bar{q}_{a'}) - \exp(\beta \bar{q}_a+\beta q_{a'})|}{\sum_{a\in\bA}\exp(\beta q_a)\sum_{a\in\bA}\exp(\beta \bar{q}_a)} \\
    \le & \beta\max_{a \in \bA}|q_a - \bar{q}_a| \frac{\sum_{a\in\bA}\sum_{a' \in \bA}[\exp(\beta q_a + \beta \bar{q}_{a'}) + \exp(\beta \bar{q}_a+\beta q_{a'})]}{\sum_{a\in\bA}\exp(\beta q_a)\sum_{a\in\bA}\exp(\beta \bar{q}_a)} \\
    =& 2\beta\max_{a\in\bA}|q_a-\bar{q}_a|,
\end{align}
where we used $|e^x-e^y| \le (e^x+e^y)|x-y|/2$. Therefore
\begin{equation}
    I_2 \le 2\beta H \sum_{h =1}^H\EE_{M,\pi^{Q^*,\beta}}\max_{a \in \bA}|Q_h^*(S_h,a) -Q_h(S_h,a)|.
\end{equation}
Combining the above estimation and inequality \eqref{eq_I_1}, we conclude our proof.
\end{proof}
\subsection{Estimation Error}
In the context of RL, sample complexity, \text{i.e.,} the number of samples required to obtain a near-optimal policy, is often used to refer to the estimation error and plays a central role in the theoretical analysis of RL.
However, in contrast to the estimation error in supervised learning, which can be characterized by the gap between the empirical loss and the population loss, the estimation error in RL is much more complex. The main challenge in the RL problem is the so-called \textit{distribution mismatch} phenomenon. Take value-based methods as an example. Assume that we have an estimation of the optimal Q-value function $Q_h^*$, denoted by $\hat{Q}_h^*$, which is close to $Q_h^*$ in the sense of $L^2(\nu)$ for a prespecified distribution $\nu$. Then, we consider the performance of the greedy policy $\hat{\pi}$ with respect to $\hat{Q}_h^*$. Using the performance difference lemma, we have
\begin{align}
    0 \le J(M,\pi^*) - J(M,\pi) &= \EE_{M,\hat{\pi}} \sum_{h=1}^H\sum_{a\in\bA}Q_h^*(S_h,a)[\pi_h^*(a\cond S_h) - \hat{\pi}_h(a\cond S_h)]\\
    &\le \EE_{M,\hat{\pi}}\sum_{h=1}^H\sum_{a\in\bA}[Q_h^*(S_h,a) - \hat{Q}_h^*(S_h,a)][\pi_h^*(a \cond S_h) - \hat{\pi}_h(a \cond S_h)],
\end{align}
where in the last inequality, we use that $\sum_{a\in\bA} \hat{Q}^*_h(s,a)[\pi_h^*(a \cond s)-\hat{\pi}_h(a \cond s)] \le 0$ since $\hat{\pi}_h$ is the greedy policy with respect to $\hat{Q}_h^*$. Therefore, we need to control the difference between $Q_h^*$ and $\hat{Q}_h^*$ under the state distribution generated by the policy $\hat{\pi}$, which is unknown before we obtain $\hat{Q}_h^*$. We refer to this phenomenon as the distribution mismatch: a mismatch between the distribution $\nu$ for estimation and the distribution for evaluation that is unknown a priori. This phenomenon is ubiquitous in the analysis of RL; see, e.g.,~\cite[Section~6]{kakade2002approximately}.

\subsection{Optimization Error}
In the context of RL, the optimization error differs significantly between value-based methods and policy-based methods. In value-based methods, the optimization error arises during the optimization process at each iteration, such as in \eqref{optimization:1} or \eqref{optimization:2} in Algorithm \ref{alg_fqi} or \ref{alg_vi}. As these optimization problems typically have a similar form to those in supervised learning, the analysis of the optimization error in RL is largely comparable to the analysis of the optimization error in supervised learning. On the other hand, the optimization error in policy-based methods, particularly the rate at which the algorithm's performance converges as the number of iterations increases, is a key focus in the theoretical analysis of these methods. The analysis of the optimization problem in policy-based methods is more challenging than in supervised learning due to the shift of the distribution of the trajectories $\{S_1^i,A_1^i,\dots,S_H^i,A_H^i\}_{i=1}^N$ in Algorithm \ref{alg_pg} during the optimization process in policy-based methods.

\section{Linear Setting}\label{sec:linear}
The simplest form of function approximation is linear function approximation. It is the setting under which the most recent theoretical results in RL are derived with function approximation. In the linear setting, we do not assume that the state space and action space are finite, and hence we need to make some structural assumptions to obtain meaningful results. The most common one is the linear MDP assumption introduced in \cite{jin2020provably}.

\begin{defn}[\textrm{Linear MDP}]
We say an MDP$(\bS,\bA,H,P,r,\mu)$ is a linear MDP with a feature map $\bm \phi:\bS\times \bA \rightarrow 
\bR^d$, if for any $h \in [H]$, there exists $d$ unknown signed measures $\bm \mu_h = (\mu_h^1,\dots,\mu_h^d)$ over $\bS$ and an unknown vector $\bm \theta_h \in \bR^d$, such that for any $(s,a) \in \bS \times \bA$
\begin{align}
    &P(\,\cdot\cond h,s,a) = \bm \phi\transpose(s,a)\bm \mu_h(\cdot),\\
    &r(h,s,a) = \bm \phi\transpose(s,a) \bm \theta_h. 
\end{align}
\end{defn}

We shall notice that the tabular MDP is a special case of the linear MDP if we set $d = |\bS||\bA|$, index each coordinate of $\bR^d$ by state-action pair $(s,a) \in \bS \times \bA$, choose $\bm\phi(s,a)$ as the canonical basis in $\bR^d$ and set
\begin{equation}\label{tabular_as_linear}
    (\bm \theta_h)_{s,a} = r(h,s,a), \,(\bm \mu_h(\cdot))_{s,a} = P(\,\cdot \cond h,s,a)
\end{equation}
for any $(h,s,a) \in [H]\times\bS\times\bA$.

\subsection{Approximation Error}
The next theorem demonstrates that the linear MDP assumption is the necessary and sufficient condition of the \ref{item:approx_Q4} holding true when discussing the approximation error in Section \ref{sec:appro}. Noticing that \ref{item:approx_Q4} is the strongest one among the four, we know that under linear MDP assumption, for any policy $\pi$, the corresponding Q-value function $Q_h^\pi$ lies in the linear space, i.e., there exist weights $\{\bm w_h^\pi\}_{h \in [H]}$ such that for any $(h,s,a) \in [H]\times\bS\times\bA$,
\begin{equation}\label{Q_h_linear}
    Q_h^\pi(s,a) = \bm \phi\transpose(s,a) \bm w_h^\pi.
\end{equation}

\begin{thm}\label{thm:lin_mdp}
Assume that $\bS\times\bA$ is a compact set and $\bm \phi: \bS\times\bA \rightarrow \bR^d$ is a feature map. Then the following two statements are equivalent. 
\begin{enumerate}
    \item There exist $d$ signed measures $\bm \mu_h = (\mu_h^1,\dots,\mu_h^d)$ over $\bS$ and a vector $\bm \theta_h \in \bR^d$, such that for any $(s,a) \in \bS\times\bA$
    \begin{align}
    &P(\,\cdot\cond h,s,a) = \bm \phi\transpose(s,a)\bm \mu_h(\cdot),\\
    &r(h,s,a) = \bm \phi\transpose(s,a) \bm \theta_h. 
\end{align}
\item For any $f \in C(\bS)$, there exists $\bm \omega_{f,h} \in \bR^d$ such that
\begin{equation}
     (\mathcal{T}_h f)(s,a) = r(h,s,a) + \EE_{s'\sim P(\,\cdot\cond h,s,a)}[f(s')] = \bm \phi\transpose(s,a)\bm \omega_{f,h}
\end{equation}
 and $\|\bm \omega_{f,h}\| \le C[\|f\|_{C(\bS)}+1]$ for a constant $C>0$, where we use $\|\cdot\|$ to denote the $l^2$-norm on $\bR^d$.
\end{enumerate}
\end{thm}

The linear MDP assumption completely ensures \ref{item:approx_Q4} (in Section \ref{sec:appro}) in the linear setting, but it rules out non-trivial deterministic MDPs that are frequently encountered in real-world situations. This is because $P(\,\cdot\cond h,s,a)$ is a delta distribution and hence the supports of $\mu_h^1,\dots,\mu_h^d$ are all single point sets. Therefore, the MDP can only visit at most $d$ state when $h \ge 2$. Hence, there are some studies that direct assume \ref{item:approx_Q2} \cite{agarwal2021theory} or \ref{item:approx_Q3} \cite{zanette2020learning}. However, there has been  little investigation into the concrete conditions of the transition probability and reward function  under which the MDP satisfies \ref{item:approx_Q2} or \ref{item:approx_Q3}.

\paragraph{}\noindent
\begin{proof}(Proof of Theorem \ref{thm:lin_mdp})
$1 \Rightarrow 2$: for any $f \in C(\bS)$, we have
\begin{equation}
    \mathcal{T}_h f = r(h,s,a) + \EE_{s'\sim P(\,\cdot\cond h,s,a)}[f(s')] = \phi\transpose(s,a)[\bm{\theta_h}+ \int_{\bS}f(s')\rmd \bm{\mu}_h(s')].
\end{equation}
Therefore, let
\begin{equation}
   \bm \omega_{f,h} = \bm{\theta}_h + \int_{\bS} f(s')\rmd \bm \mu_h(s').
\end{equation}
Hence,
\begin{equation}
    \|\bm \omega_{f,h}\| \le \|\bm{\theta}_h\| + \||\bm\mu_h|_{TV}\|\|f\|_{C(\bS\times\bA)},
\end{equation}
where $|\bm \mu_h|_{TV} = (|\mu_h^1|_{TV},\dots,|\mu_h^d|_{TV})$ is the total variation of signed measures $\bm \mu_h$. We can then choose $C = \max\{\|\bm \theta_h\|,\||\mu|_{TV}\|\}$.

$2\Rightarrow 1$: 
Let $f = 0$, we know that we can choose $\bm \theta_h = \bm \omega_{0,h}$ such that $r(h,s,a) = \bm \phi\transpose(s,a)\theta_h$. Moreover $\|\bm \theta_h\| \le C$.

Let 
\begin{equation}
    \mathcal{W}_0 = \{\bm \omega \in \bR^d: \bm \phi\transpose (s,a)\cdot \bm \omega = 0, \forall (s,a) \in \bS\times \bA \}.
\end{equation}
Noticing that $\mathcal{W}_0$ is a subspace of $\bR^d$, we can define $\mathcal{W}$ as the orthogonal complement of $\mathcal{W}_0$. For any $f \in C(\bS)$, let $\bm \omega'_{f,h}$ be the orthogonal projection of $\bm \omega_{f,h}$ to $\mathcal{W}$. Then by the definition of $\mathcal{W}_0$, we have
\begin{equation}
    (\mathcal{T}_h f)(s,a) = \bm \phi \transpose(s,a)\bm\omega_{f,h} = \bm \phi\transpose(s,a)\bm\omega'_{f,h}.
\end{equation}
On the other hand, for any $\bm \omega \in \mathcal{W}$ such that $(\mathcal{T}_h f)(s,a) = \bm \phi \transpose(s,a) \bm \omega$ holds for any $(s,a) \in \bS\times \bA$, by the definition of $N_0$, we know that $\omega = \omega'_{f,h}$. We can then define a mapping from $f$ to $\bm\omega'_{f,h}$ and 
$\mathcal{B}: C(\bS\times\bA)\rightarrow N$:
\begin{equation}
    \mathcal{B} f = \bm \omega'_{f,h} - \bm \theta'_h,
\end{equation}
where $\bm \theta'_h$ is the orthogonal projection of $\bm \theta_h$ to $\mathcal{W}$.
Then,
\begin{equation}
    \EE_{s' \sim P(\,\cdot\cond h,s,a)}[f(s')] = \phi\transpose(s,a) (\mathcal{B}f).
\end{equation}
We can then prove that $\mathcal{B}f$ is a linear mapping and
\begin{equation}
    \|\mathcal{B} f\| \le \|\bm\omega'_{f,h}\| + \|\bm \theta'_h\| \le \|\bm \omega_{f,h}\| + \|\bm \theta_h\| \le C(\|f\|_{C(\bS\times\bA)}+2)
\end{equation}
Then for any $\|f\|_{C(\bS\times\bA)} = 1$, we have $\|\mathcal{B}f\| \le 3C$, which means that 
\begin{equation}
    \|\mathcal{B}f\| \le 3C\|f\|_{C(\bS\times\bA)}.
\end{equation}
Therefore, $\mathcal{B} f$ is a bounded linear mapping from $C(\bS\times\bA)$ to $\mathcal{W}$. Noticing that $\mathcal{W}$ is a finite-dimensional linear space, we can use the Risez representation theorem on $C(\bS\times\bA)$ \cite{riesz1909operations} to show that there exists $d$ signed measure $\bm \mu_h = (\mu_h^1,\dots,\mu_h^d)$ such that for any $f \in C(\bS\times\bA)$,
\begin{equation}
    \EE_{s' \sim P(\,\cdot\cond h,s,a)}[f(s')] = \bm \phi\transpose(s,a)(\mathcal{B} f) = \bm \phi \transpose(s,a)  \int_{\bS} f(s')\rmd \bm \mu_h(s') = \int_{\bS} f(s')\rmd \bm \phi\transpose(s,a) \mu_h(s'),
\end{equation}
which means that
\begin{equation}
    P(\,\cdot \cond h,s,a) = \bm \phi\transpose(s,a) \bm \mu_h(\cdot).
\end{equation}
\end{proof}

\subsection{Estimation and Optimization Error}
Under the linear MDP assumption, the above analysis justifies our use of a linear function to approximate the Q-value function. We can then apply the fitted Q-iteration algorithm or value iteration methods in Section \ref{value_base_method} with linear function approximation. We can also use the linear function to approximate the Q-value function and then use the softmax policy to approximate the optimal policy to apply the policy gradient algorithms in Section \ref{policy_based_method}. Here we take the value iteration (Algorithm \ref{alg_vi}) in the episodic setting \cite{jin2020provably} as an example to discuss the estimation error. In \cite{jin2020provably}, they set
\begin{align}
    &\mathcal{F}_h = \{\bm \phi\transpose(s,a)\theta, \theta \in \bR^d\},\\
    &\Lambda_h(f) = \lambda \|\theta\|^2, \\
    &b_h^{k+1}(s,a) = \beta[\bm \phi(s,a)\transpose(\sum_{i=1}^k\bm \phi(S_h^i,A_h^i)\bm \phi\transpose(S_h^i,A_h^i) + n\lambda \mathrm{I})^{-1}\bm \phi(s,a)]^{\frac{1}{2}}.
\end{align}
Intuitively speaking, the bonus term $b_h^{k+1}$ is proportional to the standard deviation of the $\hat{Q}_h^{k+1}$. It is proved that \cite[Lemma~B.5]{jin2020provably}, with high probability,
\begin{equation}
    Q_h^{*}(s,a) \le Q_h^{k+1}(s,a), \, \forall (s,a) \in \bS\times\bA.
\end{equation}
In this sense, $Q_h^{k+1}$ is the uniformly upper confidence bound of $Q_h^*$. The introduction of the bonus term $b_h^{k+1}$ and resulting UCB estimation is the crucial step to address the distribution mismatch phenomenon. By properly choosing the parameters $\beta$ and $\lambda$, \cite{jin2020provably} prove that Algorithm \ref{alg_vi} can achieve the following regret bounds (recalling that $K$ denotes the number of episodes)
\begin{equation}
    \mathrm{Regret}(K) = \tilde{O}(\sqrt{d^3H^4K}),
\end{equation}
which implies that to obtain an $\epsilon$-optimal policy, the algorithm needs at most $\tilde{O}(\frac{d^3H^5}{\epsilon^2})$ samples. Such a result is independent of the number of states and actions and is much more general than the tabular setting.

In Algorithm \ref{alg_vi}, the optimization problem \eqref{optimization:2} is a ridge regression problem whose solution can be exactly computed, so there is no error introduced in this optimization step. However, we need to compute $\max_{a \in \bA} Q_h^k(s,a)$ many times. If $|\bA|$ is finite, the maximum can be exactly computed. However, if $|\bA|$ is infinite, numerical errors may be introduced when calculating the maximum value.

There are other works in the RL literature studying RL problems in the linear setting. \cite{yang2019sample,lattimore2020learning,wang2021sample} consider the linear setting with a generative model in the time-homogeneous case. These works use $L^\infty$ estimation instead of UCB estimation to handle the distribution mismatch phenomenon. \cite{zanette2020learning} considers a similar assumption with \cite{jin2020provably} but provides an algorithm with a tighter regret bound. \cite{zhou2021provably} also considers the RL problem in a linear setting but with respect to discounted MDP with infinite horizons. \cite{cai2020provably,agarwal2020pc,agarwal2021theory} study the policy-based methods for the RL problem in the linear setting. 
Most of these results establish the polynomial sample complexity with respect to the number of features $d$, the length of the episode $H$, and the accuracy $\epsilon$ under similar assumptions. However, so far, the gap between the lower bound and upper bound in sample complexity still exists in the linear setting, except for results in \cite{yang2019sample,wang2021sample}, which require a generative model and a very restrictive assumption called anchor state-action pairs assumption. This gap is evident by noticing that the tabular MDP is a special case of the linear MDP, and the lower bound in the tabular MDP implies a naive lower bound $\frac{dH^3}{\epsilon^2}$; see \cite[Section~5]{jin2020provably} for a detailed discussion.
Bridging the gap is an important direction for future work.

\section{Nonlinear Setting}\label{sec:nonlinear}
\subsection{RKHS, NTK and Barron Space}\label{sec:space}
As powerful function approximation tools (particularly in high dimensions), kernel functions and neural networks are now widely used in various machine learning tasks, including RL problems. Theoretical analysis of RL algorithms involving function approximations hinges on the proper choice of function spaces and a deep understanding of these spaces.
In this subsection, we will briefly introduce the concepts of reproducing kernel Hilbert space (RKHS), neural tangent kernel (NTK), and Barron space, as they are suitable function spaces associated with kernel function and neural network approximation. In particular, we introduce the theoretical results of supervised learning algorithms in these function spaces, which will be the foundation for analyzing RL algorithms with function approximation.

\paragraph{RKHS.} Suppose $k$ is a continuous positive definite kernel that satisfies:
\begin{enumerate}
    \item $k(x,x') = k(x',x), \forall x,x' \in \mathcal{X}$;
    \item $\forall m \ge 1$, $x_1,\dots,x_m \in \mathcal{X}$ and $a_1,\dots,a_m \in \bR$, we have:
    \begin{equation}
        \sum_{i=1}^m\sum_{j=1}^ma_ia_jk(x_i,x_j) \ge 0.
    \end{equation}
\end{enumerate}
Then, there exists a Hilbert space  $\mathcal{H}_k \subset C(\mathcal{X})$ such that
\begin{enumerate}
    \item $\forall x \in \bR^d$, $k(x,\,\cdot\,) \in \mathcal{H}_k$;
    \item $\forall x \in \bR^d$ and $f \in \mathcal{H}_k$, $f(x) = \langle f, k(x,\,\cdot\,)\rangle_{\mathcal{H}_k}$.
\end{enumerate}
$k$ is called the reproducing kernel of $\mathcal{H}_{k}$ \cite{aronszajn1950theory}, and we use $\|\cdot\|_{\mathcal{H}_k}$ to denote the norm of the Hilbert space $\mathcal{H}_k$. Common examples of reproducing kernels include the Gaussian kernel $k(x,x') = \exp(-\alpha\|x-x'\|^2)$ and the Laplacian kernel $k(x,x') = \exp(-\alpha\|x-x'\|)$ $(\alpha>0)$.

The kernel method can efficiently learn functions in the RKHS with finite data. 
In the kernel method, we set the hypothesis space as $\mathcal{H}_k$, the entire RKHS. In this sense, since the target function $f^*$ lies in $\mathcal{H}_k$, the approximation error is zero. We can then define the loss function
\begin{equation}
    \mathcal{L}_n(f) = \frac{1}{n}\sum_{i=1}^n|y_i - f(x_i)|^2 + \lambda \|f\|_{\mathcal{H}_k}^2.
\end{equation}
We can then obtain the kernel ridge estimator $\hat{f} = \argmin_{f \in \mathcal{H}_k} \mathcal{L}_n(f)$. If we choose $\lambda = n^{-\frac{1}{2}}$, then the estimation error 
\begin{equation}
    \|f^* - \hat{f}\|_{L^2(\mu)} \le O([1+\|f\|_{\mathcal{H}_k}]n^{-\frac{1}{4}}).
\end{equation}
We can then show that the kernel ridge estimator can be computed exactly, and hence the optimization error is zero. First, using the Proposition 4.2 in \cite{paulsen2016introduction}, we know that
\begin{equation}
    \min_{f \in \mathcal{H}_k} \mathcal{L}_n(f) = \min_{f = \sum_{i=1}^n \alpha_i k(\cdot,x_i)}\mathcal{L}_n(f).
\end{equation}
Then we only need to compute the $\alpha_1,\dots,\alpha_n$ to minimize the loss function. Moreover, if $f = \sum_{i=1}^n\alpha_i k(\cdot,x_i)$, then 
\begin{equation}
    \mathcal{L}_n(f) = \frac{1}{n}[\bm y - K_n \bm \alpha]\transpose[\bm y - K_n \bm \alpha] + \lambda \alpha\transpose K_n \bm \alpha,
\end{equation}
where $\bm y = (y_1,\dots,y_n)\transpose$, $\bm \alpha = (\alpha_1,\dots,\alpha_n)\transpose$ and $K_n = (k(x_i,x_j))_{1\le i,j \le n}$ is an $n\times n$ matrix. Therefore,
\begin{equation}
    \bm \hat{\alpha} = (K_n + \lambda n \mathrm{I}_d)^{-1}\bm y,
\end{equation}
which can be computed directly. See \cite{caponnetto2007optimal,steinwart2009optimal,rudi2017falkon} for more details on the kernel method.


\paragraph{NTK.} Neural tangent kernel (NTK) was first introduced to study the overparameterized neural networks (\cite{jacot2018neural}). For the input data $x\in \bR^d$, we consider a two-layer ReLU neural network with $m$ neurons:
\begin{equation}
    f(x;\theta) = \frac{1}{\sqrt{m}}\sum_{i=1}^ma_i\sigma(\omega_i\transpose x),
\end{equation}
where $\sigma (x)= \max \{x,0\}$ is the ReLU activation function. Here $\theta$ denotes the collection of all the parameters $(a_1,\omega_1,\dots,a_m,\omega_m)$, with $a_i\in\bR, \omega_i\in \bR^d, i=1,\dots,m$. If we initialize $\theta$ according to the following rule:
\begin{equation}\label{NTK_scale}
    a_i \stackrel{i.i.d.}{\sim} \mathcal{N}(0,1),\, \omega_i \stackrel{i.i.d.}{\sim} \mathcal{N}(0,\mathrm{I}_d/d),
\end{equation}
then the fully trained neural networks approximate the kernel ridge regression on the RKHS with respect to the NTK $k_{\textrm{NTK}}$ when the width $m$ goes to the infinity \cite{arora2019exact}, where
\begin{equation}
    k_{\textrm{NTK}}(x,x') = \EE_{\omega \sim \mathcal{N}(0,\mathrm{I}_d/d)}[x\transpose x'\sigma'(\omega\transpose x)\sigma'(\omega\transpose x')+ \sigma(\omega\transpose x)\sigma(\omega\transpose x')]
\end{equation}
and $\sigma'(x)= \mathrm{1}_{x > 0}$ is the derivative of $\sigma$. The theory of NTK establishes a connection between neural network and kernel methods and shows that it is sufficient to study the corresponding NTK if we are interested in the overparameterized neural networks with the NTK scaling \eqref{NTK_scale}. Therefore, we will not discuss RL with neural network approximation under the NTK regime, as it is covered by the kernel function approximation.
For more details on the NTK theory, including the NTK corresponding to multi-layer neural networks, see \cite{jacot2018neural,arora2019fine,allen2019convergence,arora2019exact,weinan2019comparative}. 

\paragraph{Barron space. }Barron space is introduced to study the overparameterized neural networks as well but with a different scaling. We consider the two-layer ReLU neural network with mean-field scaling:
\begin{equation}
     f(x;\theta) = \frac{1}{m}\sum_{i=1}^ma_i\sigma(\omega_i\transpose x).
\end{equation}
The function in the Barron space serves as the continuous analog of the two-layer neural network as the width $m$ goes to infinity:
\begin{equation}\label{def_Barron_function}
   f(x) =  \int_{\bR^d}a(\omega)\sigma(\omega\transpose x)\rmd \rho(\omega).
\end{equation}
The Barron space is defined as follows
\begin{align}
    \mathcal{B} = \bigg\{ f(x) = \int_{\bR^d}a(\omega)\sigma(\omega\transpose x)\rmd\rho(\omega),~\rho \in \mathcal{P}(\bR^d),
    \text{ and } \int_{\bR^d}|a(\omega)|\rmd \rho(\omega) < +\infty \bigg\}.
\end{align}
Due to the scaling invariance of the ReLU function, the norm of the Barron space can be defined by
\begin{equation}
    \|f\|_\mathcal{B} = \inf_{a,\rho} \int_{\mathbb{S}^{d-1}}|a(\omega)|\rmd \rho(\omega),
\end{equation}
where the infimum is taken over all possible $\rho \in \mathcal{P}(\mathbb{S}^{d-1})$ and $a \in L^1(\rho)$  such that Eq.~\eqref{def_Barron_function} is satisfied. We have the following relationship between the RKHS and Barron space:
\begin{equation}\label{barron_rkhs_relation}
    \mathcal{B} =  \mathop{\cup}_{\pi \in \mathcal{P}(\mathbb{S}^{d-1})} \mathcal{H}_{k_\pi},
\end{equation}
where $k_\pi(x,y) = \EE_{\omega\sim\pi} [\sigma(\omega\transpose x)\sigma(\omega\transpose y)]$. We shall also point out that compared to the RKHS, the Barron space is much larger in high dimensions, see Example 4.3 in \cite{weinan2021kolmogorov}. We refer to \cite{ma2019barron,ma2022barron} for more details and properties of the Barron space. For a detailed comparison between the NTK and the Barron space, see \cite{weinan2019comparative,ma2020towards}.

We can approximate the target function in Barron space with the two-layer neural networks:
\begin{equation}
    \mathcal{H}_m = \{f(x,\theta) = \frac{1}{m}\sum_{i=1}^m a_i \sigma(\omega_i\transpose x): \theta = (a_1,\omega_1,\dots,a_m,\omega_m), a_i \in \bR, \omega_i \in \mathbb{S}^{d-1}, 1 \le i \le m\}.
\end{equation}
Then we have that for any $f^* \in \mathcal{B}$ and $\mu \in \mathcal{P}(\bR^d)$:
\begin{equation}
    \inf_{f_m \in \mathcal{H}_m}\EE_{x \sim \mu}|f^*(x) - f_m(x)|^2 \le \frac{\|f\|_\mathcal{B}^2}{m}.
\end{equation}

Similar to the kernel ridge regression, we can define the following loss function:
\begin{equation}
    \mathcal{L}_n(\theta) = \frac{1}{n}\sum_{j=1}^n \max\{(\ln n)^2, [y_i - f(x_i,\theta)]^2\} + \frac{\lambda}{m}\sum_{i=1}^m|a_i|\|\omega_i\|
\end{equation}
to obtain the estimator $\hat{f}$. Under proper condition, we can obtain that 
\begin{equation}
    \|f^*-\hat{f}\|_{L^2(\mu)} = \tilde{O}(\|f\|_\mathcal{B}[m^{-\frac{1}{2}} + n^{-\frac{1}{4}}]),
\end{equation}
 if we set $\lambda = \tilde{O}(n^{-\frac{1}{2}})$, see \cite{ma2018priori} for details. However, it is still not clear how to efficiently compute the estimator $\hat{f}$, see, e.g., \cite{weinan2019comparative,ma2020towards}. More work is needed to further understand the optimization error in this case.

\subsection{Reinforcement Learning with Nonlinear Function Approximation}

We first discuss the approximation error in the nonlinear setting. To this end, we generalize Theorem \ref{thm:lin_mdp} to the case of RKHS.
\begin{thm}\label{thm:rkhs}
    Assume that $\bS\times\bA$ is a compact set. Let $\rho$  be a probability distribution on $\bS \times \bA$, we will use $\{\lambda_i\}_{i=1}^{+\infty}$ and $\{\psi_i\}_{i=1}^{+\infty}$ to denote the eigenvalues and eigenfunctions of the operator
\begin{equation*}
    (\mathcal{K}_{\rho}g) (x) \coloneqq \int_{\bS \times \bA}k(x,x')g(x')\rmd \rho(z') 
\end{equation*}
from $L^2(\rho)$ to $L^2(\rho)$. We further require that $\{\lambda_i\}_{i=1}^{+\infty}$ is nonincreasing and $\{\psi_i\}_{i=1}^{+\infty}$ is orthonormal in $L^2(\rho)$. 
Then, the following statements are equivalent.
\begin{enumerate}
    \item For any $h \in [H]$, $r(h,\cdot) \in \mathcal{H}_k$ and there exist signed measures $\{\mu_h^i\}_{i=1}^{+\infty}$ over $\bS$, such that for any $(s,a)$
    \begin{equation}
        P(\,\cdot\cond h,s,a) = \sum_{i=1}^{+\infty}\psi_i(s,a) \mu_h^i(\cdot)
    \end{equation}
    in the sense that for any $f \in C(\bS)$
    \begin{equation}
        \int_{\bS}f(s')\rmd P(s'|h,s,a) = \sum_{i=1}^{+\infty}\int_{\bS}f(s')\rmd \mu_h^i(s') \psi_i(s,a),
    \end{equation}
    where the convergence is in the sense of $L^2(\rho)$. Moreover, for any $f \in C(\bS)$
    \begin{equation}
        \sum_{i=1}^{+\infty}\frac{1}{\lambda_i}|\int_{\bS} f(s') \rmd \mu_h^i(s')|^2 \le C\| f \|_{C(\bS)}^2,
    \end{equation}
    for a constant $C > 0$.
    \item For any $f \in C(\bS)$ and $h \in [H]$,
    \begin{equation}
        (\mathcal{T}_h f)(s,a) = r(h,s,a) + \EE_{s' \sim P(\,\cdot\cond h,s,a)}[f(s')] \in \mathcal{H}_k
    \end{equation}
    and $\|\mathcal{T}_h f\|_{\mathcal{H}_k} \le C'[\|f\|_{C(\bS)} + 1]$ for a constant $C' > 0$.
\end{enumerate}

\end{thm}
\begin{proof}
We will use the following representation of the RKHS norm: 
\begin{equation}\label{mercer_norm}
    \|g\|_k^2 = \sum_{i=1}^{+\infty}\frac{1}{\lambda_i}|\langle g, \psi_i\rangle_{L^2(\rho)}|^2.
\end{equation}
See, e.g., \cite[Section~2.1]{bach2017equivalence}.

$1 \Rightarrow 2$: For any $f \in C(\bS)$, 
\begin{equation}
    (\mathcal{T}_h f)(s,a) = r(h,s,a) + \sum_{i=1}^{+\infty} \int_{\bS} f(s')\rmd \mu_h^i(s') \psi_i(s,a). 
\end{equation}
Let 
\begin{equation}
    g_{f,h} = \sum_{i=1}^{+\infty} \int_{\bS} f(s')\rmd \mu_h^i(s') \psi_i,
\end{equation}
then $\langle g_{f,h}, \psi_i\rangle_{L^2(\rho)} = \int_{\bS} f(s') \rmd \mu_h^i(s')$ for any $i \in \mathbb{N}^+$. Therefore $g_{f,h} \in \mathcal{H}_k$ and 
\begin{equation}
    \|g_{f,h}\|_{\mathcal{H}_k} = \sqrt{\sum_{i=1}^{+\infty}\frac{1}{\lambda_i}|\int_{\bS} f(s')\rmd \mu_i(s')|^2} \le \sqrt{C}\|f\|_{C(\bS)}.
\end{equation}
Hence
\begin{equation}
    \|\mathcal{T}_h f\|_{\mathcal{H}_k} \le \|r(h,\cdot)\|_{\mathcal{H}_k} + \|g_{f,h}\|_{\mathcal{H}_k} \le C'[1+ \|f\|_{C(\bS)}].
\end{equation}

$2 \Rightarrow 1$: Choose $f = 0$ we know that $r(h,\cdot) \in \mathcal{H}_k$ and $\|r(h,\cdot)\|_{\mathcal{H}_k} \le C'$. Let
\begin{equation}
    g_{f,h} = \EE_{s'\sim P(\,\cdot\cond h,s,a)}[f(s')],
\end{equation}
then for any $f \in C(\bS)$, $g_{f,h} \in \mathcal{H}_k$. If $\|f\|_{C(\bS)} = 1$, we have
\begin{equation}
    \|g_{f,h}\|_{\mathcal{H}_k} \le \|\mathcal{T}_h f\|_{\mathcal{H}_k} + \|r(h,\cdot)\|_{\mathcal{H}_k} \le 3C'.
\end{equation}
Then by the linearity of $g_{f,h}$ with respect to $f$, we have that for any $f \in C(\bS)$,
\begin{equation}
    \|g_{f,h}\|_{\mathcal{H}_k} \le 3C'\|f\|_{C(\bS)}.
\end{equation}
Noticing that $g_{f,h} \in \mathcal{H}_k \subset L^2(\rho)$, we know that for any $f \in C(\bS)$,
\begin{equation}
   \int_{\bS} f(s')\rmd P(s'|h,s,a) = g_{f,h} = \sum_{i=1}^{+\infty} \int_{\bS\times\bA}g_{f,h}(s',a') \psi_i(s',a')\rmd \rho(s',a') \psi_i(s,a).
\end{equation}
Noticing that
\begin{equation}
    |\int_{\bS\times\bA}g_{f,h}(s',a') \psi_i(s',a')\rmd \rho(s',a')|\le \|g_{f,h}\|_{L^2(\rho)} \le \sqrt{\lambda_1}\|g_{f,h}\|_{\mathcal{H}_k} \le 3C'\sqrt{\lambda_1}\|f\|_{C(\bS)}. 
\end{equation}
We can then use the Riesz representation theorem to obtain that there exists signed measures $\{\mu_h^i\}_{i=1}^{+\infty}$ such that
\begin{equation}
    \int_{\bS}g_{f,h}(s',a') \psi_i(s',a')\rmd \rho(s',a') = \int_{\bS} f(s')\rmd \mu_h^i(s').
\end{equation}
Therefore,
\begin{equation}
    P(\,\cdot\cond h,s,a) = \sum_{i=1}^{+\infty}\mu_h^i(\cdot)\psi_i(s,a).
\end{equation}
Moreover,
\begin{equation}
    \|g_{f,h}\|_{\mathcal{H}_k}^2 = \sum_{i=1}^{+\infty}\frac{1}{\lambda_i}|\int_{\bS}f(s')\rmd \mu_h^i(s')|^2  \le 9(C')^2\|f\|_{C(\bS)}^2 := C\|f\|_{C(\bS)}^2.
\end{equation}
\end{proof}

Similar to the discussions in the linear setting, Theorem \ref{thm:rkhs} provides a necessary and sufficient condition to ensure \ref{item:approx_Q4} in the RKHS setting. The extension of Theorem \ref{thm:rkhs} to the Barron space is not yet clear, mainly due to the lack of a representation of the Barron norm that is similar to \eqref{mercer_norm}. 
Nevertheless, it is still worthwhile to consider only sufficient conditions for \ref{item:approx_Q4} in Barron space. \cite{long20212} introduce one such condition:
for any $h \in [H]$,
\begin{equation}\label{barron_approx}
    \|r(h,\cdot)\|_{\mathcal{B}} < +\infty,\quad \sup_{s' \in \bS}\|p(h,s',\cdot)\|_\mathcal{B} < +\infty,
\end{equation}
where $p(h,s',s,a) = \frac{\rmd P(s'\cond h,s,a)}{\rmd \rho_h(s')}$ and $\rho_h$ is a probability distribution on $\bS$. Similar to the situation in the linear setting, these conditions rule out many interesting deterministic MDPs, since these MDPs can only visit countably infinite states when $h \ge 2$. Additionally, there has been little investigation into the concrete conditions on the transition probability and reward function that ensure \ref{item:approx_Q2} and \ref{item:approx_Q3} in Section \ref{sec:appro} in the nonlinear setting.

With theoretical results in supervised learning and analysis on the approximation error of RL in the nonlinear setting, it remains to develop tools to handle the distribution mismatch phenomenon in the nonlinear setting, which leads to a paramount difference in RL analysis between tabular/linear settings and nonlinear settings. In the tabular and linear settings, the $L^\infty$ estimation and UCB estimation are used to handle the distribution mismatch. The $L^\infty$ estimation or UCB estimation 
of the value function is obtained such that the error under any distribution can be controlled. However, as pointed out in \cite{kuo2008multivariate}, \cite{long20212} and the theorem below, both $L^\infty$ and UCB estimations will suffer from the curse of dimensionality for high-dimensional NTK, Barron space, and many common RKHSs.  This challenge reveals at least one essential difficulty of RL problems in the nonlinear setting compared to the tabular and linear settings.

\begin{thm}
\label{thm:Linfinity}
Given an RKHS $\mathcal{H}_k$ on $\mathcal{X}$ associated with a continuous kernel $k$ (assuming that $\sup_{x \in \mathcal{X}}k(x,x) \le 1$) and any $x_1,\dots,x_n \in \mathcal{X}$, let $\mathcal{H}_k^1$ be the unit ball of $\mathcal{H}_k$ and $\mathcal{G}_n: \mathcal{H}_k^1 \rightarrow C(\mathcal{X})$ be a mapping satisfying
\begin{equation}
    \mathcal{G}_n f = \mathcal{G}_n f',\, \forall f, f' \in \mathcal{H}_k^1 \text{ such that } f(x_i) = f'(x_i), i=1,\dots,n.
\end{equation}
For any given distribution $\rho$ on $\mathcal{X}$, let $\{\lambda_i\}_{i=1}^
{+\infty}$ be the nonincreasing eigenvalues of the mapping $K_\rho: L^2(\rho) \rightarrow L^2(\rho)$:
\begin{equation}
    (K_\rho g)(x) = \int_{\mathcal{X}}k(x,x') g(x') \rmd \rho(x').
\end{equation}
The following two statements hold true.
\begin{enumerate}
    \item{($L^\infty$ estimation)} $\displaystyle \sup_{f \in \mathcal{H}_k^1}\|f - \mathcal{G}_n f\|_\infty \ge (\sum_{i=n+1}^{+\infty} \lambda_i)^{\frac{1}{2}}$.
    \item{(UCB estimation)} If $\mathcal{G}_n$ additionally satisfies that 
    \begin{equation}\label{ucb_condition}
    \mathcal{G}_n f(x) \ge f(x), \,\forall f \in \mathcal{H}_k^1 \text{ and } x \in \mathcal{X},
    \end{equation}
    then,
    \begin{equation}
       \sup_{f \in \mathcal{H}_k^1} \EE_{\rho}[\mathcal{G}_n f - f] \ge \sum_{i=n+1}^{+\infty}\lambda_i. 
    \end{equation}
\end{enumerate}
\end{thm}

We can interpret the mapping $\mathcal{G}_n$ in Theorem~\ref{thm:Linfinity} as an abstraction of a function approximation algorithm that takes the function values at $n$ points, $x_1,\dots,x_n$ as input and returns a continuous function. The requirement in Theorem~\ref{thm:Linfinity} is naturally satisfied: if two target functions have the same values at $x_1,\dots,x_n$, the function approximation result will be identical. From the definition, UCB estimation must satisfy condition \eqref{ucb_condition} as the UCB estimation should give a pointwise upper bound of the target function. Therefore, Theorem \ref{thm:Linfinity} gives the lower bound of the worst-case error of both $L^\infty$ and UCB estimation based on the eigenvalue decay of the kernel. As pointed out in~\cite{long20212},
if we choose $\mathcal{X} = \mathbb{S}^{d-1}$, the unit ball in $\bR^d$ and $\rho$ the uniform distribution on $\mathbb{S}^{d-1}$, the eigenvalue decay $\sum_{i=n+1}^{+\infty}\lambda_i$ of the following RKHSs: 
\begin{equation}
    k(x,x') = \begin{cases} k_{Lap}(x,x') = \exp(-\|x - x'\|)  \\
    k_{NTK}(x,x') = \EE_{\omega \sim \pi}( x \cdot x')\sigma'(\omega \cdot x)\sigma'(\omega \cdot x')  \\
    k_\pi(x,x') = \EE_{\omega \sim \pi}\sigma(\omega \cdot x)\sigma(\omega\cdot x')
    \end{cases}
\end{equation}
is $n^{-\frac{\alpha}{d}}$ for some universal constant $\alpha$. Here $\pi$ is also the uniform distribution on $\mathbb{S}^{d-1}$. Therefore, if the target function lies in the RKHS associated with the Laplacian kernel or NTK, according to the first argument in Theorem~\ref{thm:Linfinity}, the $L^\infty$ estimation suffers from the curse of dimensionality: the number of points needed to achieve an error tolerance scales exponentially with respect to the dimension $d$.
Since the $\mathcal{H}_{k_\pi}$ is the subspace of the Barron space $\mathcal{B}$ (equation \eqref{barron_rkhs_relation}), the $L^\infty$ and UCB estimation in the Barron space also suffer from the curse of dimensionality. 

\paragraph{}\noindent
\begin{proof}(of Theorem \ref{thm:Linfinity})
We first prove that
\begin{equation}\label{eigen_control_average}
    \EE_{x\sim\rho}\sup_{f \in \mathcal{H}_k^1, f(x_1) = \dots = f(x_n) = 0}|f(x)|^2 \ge \sum_{l = n+1}^{+\infty}\lambda_l.
\end{equation}

 Notice that
\begin{align}
    \sup_{f \in \mathcal{H}_k^1, f(x_1) = \dots = f(x_n) = 0}f(x) &= \sup_{\|f\|_\mathcal{H} \le 1, \langle f,k(x_i,\cdot)\rangle_{\mathcal{H}_k} = 0, 1\le i \le n}\langle f, k(x,\cdot)\rangle_{\mathcal{H}_k}\\
    &=\inf_{c_1,\dots,c_n}\|k(x,\cdot)-\sum_{i=1}^nc_ik(x_i,\cdot)\|_{\mathcal{H}_k}.
\end{align}
Then, let $\phi_1,\dots,\phi_n$ be the Gram-Schmidt orthonormalization of $\{k(x_1,\cdot),\dots,k(x_n,\cdot)\}$ in $\mathcal{H}$, then
\begin{equation}
    \inf_{c_1,\dots,c_n}\|k(x,\cdot)-\sum_{i=1}^nc_ik(x_i,\cdot)\|_{\mathcal{H}_k}^2 =  k(x,x) -\sum_{i=1}^n\phi_i^2(x).
\end{equation}
Therefore,
\begin{align}
    \EE_{x\sim\rho}\sup_{f \in \mathcal{H}_k^1, f(x_1) = \dots = f(x_n) = 0}|f(x)|^2 = \EE_{x\sim\rho}k(x,x) - \sum_{i=1}^n\EE_{x\sim\rho}\phi_i^2(x).
\end{align}
Let $\{\psi_l\}_{l=1}^{+\infty}$ be the eigenfunctions corresponding to $\{\lambda_l\}_{l=1}^{+\infty}$, which is an orthonormal basis in $L^2(\rho)$. Let $c_l = \sum_{i=1}^n(\EE_{x\sim\rho}\psi_l(x)\phi_i(x))^2$,
then,
\begin{equation}
   \lambda_l = \lambda_l^2\|\psi_l\|_{\mathcal{H}_k}^2 \ge \lambda_l^2\sum_{i=1}^n(\langle \psi_l,\phi_i\rangle_{\mathcal{H}_k})^2 = \lambda_l^2\sum_{i=1}^n\frac{(\EE_{x\sim\rho}\psi_l(x)\phi_i(x))^2}{\lambda_l^2} = c_l \ge 0,
\end{equation}
and
\begin{equation}
    \sum_{l=1}^{+\infty}\frac{c_l}{\lambda_l} = \sum_{i=1}^n\sum_{l=1}^{+\infty}\frac{(\EE_{x\sim\rho}\psi_l(x)\phi_i(x))^2}{\lambda_l} = \sum_{i=1}^n\|\phi_i\|_\mathcal{H}^2 = n.
\end{equation}
Hence,
\begin{equation}
    \sum_{i=1}^n\EE_{x\sim\rho}\phi_i^2(x) = \sum_{l=1}^{+\infty}c_l \le \sum_{l=1}^n\lambda_l.
\end{equation}
The famous Mercer decomposition states that
\begin{equation}\label{mercer_decompo}
    k(x,x') = \sum_{i=1}^{+\infty}\lambda_i \psi_i(x)\psi_i(x').
\end{equation}
Therefore, with the observation that
\begin{equation}
    \EE_{x\sim\rho}k(x,x) = \sum_{l=1}^{+\infty}\lambda_l,
\end{equation}
we obtain inequality \eqref{eigen_control_average}.

To prove the first argument, first notice that
\begin{align}
      &\sup_{f \in \mathcal{H}_k^1, f(x_1) = \dots=f(x_n)=0}\|f\|_\infty 
    = \sup_{x \in \mathcal{X}} \sup_{f \in \mathcal{H}_k^1, f(x_1) = \dots=f(x_n)=0}|f(x)| \\
    \ge  &(\EE_{x\sim\rho}\sup_{f \in \mathcal{H}_k^1, f(x_1) = \dots = f(x_n) = 0}|f(x)|^2)^{\frac{1}{2}} \ge (\sum_{l = n+1}^{+\infty}\lambda_l)^{\frac{1}{2}}.
\end{align}
Then, noticing that for any $f\in \mathcal{H}_k^1$ such that $f(x_1) = \dots = f(x_n) =0 $, we have $\mathcal{G}_n f = \mathcal{G}_n(-f)$. Therefore
\begin{align}
    \sup_{f \in \mathcal{H}_k^1}\|f - \mathcal{G}_n f\|_\infty &= \sup_{f \in \mathcal{H}_k^1, f(x_1) = \dots=f(x_n)=0}\frac{\|f-\mathcal{G}_nf\|_\infty + \|-f -\mathcal{G}_nf\|_\infty}{2}\\
    &\ge \sup_{f \in \mathcal{H}_k^1, f(x_1) = \dots=f(x_n)=0}\|f\|_\infty,
\end{align}
which concludes the proof.

For the second argument, let $f_0 = 0$. For any $f \in \mathcal{H}_k^1$ such that $f(x_1) = \dots = f(x_n) = 0$, we have
\begin{equation}
    \mathcal{G}_n f_0(x) = \mathcal{G}_n f(x) \ge f(x),
\end{equation}
for any $x \in \mathcal{X}$. Therefore,
\begin{equation}
    \mathcal{G}_nf_0(x) \ge \sup_{f\in\mathcal{H}_k^1,f(x_1)=\dots = f(x_n)=0}f(x) = \sup_{f\in\mathcal{H}_k^1,f(x_1)=\dots = f(x_n)=0}|f(x)|.
\end{equation}
Combining the fact that
\begin{equation}
    \sup_{f \in \mathcal{H}_k^1}|f(x)| = \sup_{f \in \mathcal{H}_k^1}|\langle f,k(x,\cdot)\rangle_{\mathcal{H}_k}| = \|k(x,\cdot)\|_{\mathcal{H}_k} = k(x,x) \le 1,
\end{equation}
we know that 
\begin{equation}
    \mathcal{G}_n f_0(x) \ge \sup_{f\in\mathcal{H}_k^1,f(x_1)=\dots = f(x_n)=0}|f(x)|^2.
\end{equation}
Therefore,
\begin{equation}
    \sup_{f\in\mathcal{H}_k^1}\EE_{x\sim\rho}[\mathcal{G}_nf - f] \ge \EE_{x\sim\rho} \mathcal{G}_n f_0(x) \ge \EE_{x\sim\rho}\sup_{f\in\mathcal{H}_k^1,f(x_1)=\dots = f(x_n)=0}|f(x)|^2 \ge \sum_{i=n+1}^{+\infty}\lambda_i.
\end{equation}
\end{proof}

Theorem~\ref{thm:Linfinity} indicates that the $L^\infty$ or UCB estimation is too strong as a requirement to pursue in the high-dimensional cases with nonlinear function approximation. Therefore, to obtain meaningful results in the nonlinear setting, besides the assumption that ensures the value or policy function can be approximated by the kernel or neural functions, some additional assumptions are needed to handle the distribution mismatch phenomenon. Based on the assumptions used, most of the existing works addressing this difficulty can be divided into two categories. The first category \cite{yang2020provably,yang2020function} assumes the fast eigenvalue decay of the kernel such that the $L^\infty$ and UCB estimation still provide a meaningful bound in high dimensions. The second category \cite{farahmand2016regularized,chen2019information,wang2019neural,fan2020theoretical,agarwal2021theory,long20212} requires the following concentration coefficient condition: for any $h \in [H]$, there exists a distribution $\nu_h$ such that for any policy $\pi$, the corresponding state-action distribution $\rho_{h,P,\pi,\mu}$ satisfies
\begin{equation}
    \|\frac{\rmd \rho_{h,P,\pi,\mu}}{\rmd \nu_h}\|_{L^2(\nu_h)} \le C,
\end{equation}
where $C > 0$ is a universal constant. Under this assumption, an $L^2$ estimation under $\nu_h$ is sufficient to handle distribution mismatch since we can control the estimation error under the state-action distributions generated by all possible policies, including the optimal policy. This assumption is commonly used to study the convergence of the fitted Q-iteration algorithm (Algorithm \ref{alg_fqi})  \cite{farahmand2016regularized,chen2019information,fan2020theoretical,long20212}.
In the episodic setting, due to the lack of a generative model, we need to additionally assume that $\nu_h$ is the state-action distribution $\rho_{h,P,\bar{\pi},\mu}$ for a policy $\bar{\pi}$ \cite{wang2019neural}.

To better capture the influence of distribution mismatch in the RL problem, \cite{long2022perturbational} introduce a quantity called \textit{perturbational complexity by distribution mismatch} for a large class of the RL problems in the nonlinear setting when a generative model is accessible. This quantity can give both the lower bound and upper bound of the sample complexity of these RL problems and hence measure their difficulty. Moreover, both fast eigenvalue decay and finite concentration coefficient can lead to small perturbational complexity by distribution mismatch \cite[Proposition~2 and 3]{long2022perturbational} and hence the results in \cite{long2022perturbational} generalize both categories of the previous results in the nonlinear setting. 

The formal definition of the perturbational complexity by distribution mismatch is given as follows.
\begin{defn}~
\begin{enumerate}[label={\upshape(\roman*)}, widest=iii]
\item 
For any set $\Pi$ consisting of probability distributions on $\bS\times\bA$, we define a semi-norm $\|\cdot\|_\Pi$ on $C(\bS\times\bA)$:
\begin{equation*}
    \|g\|_\Pi \coloneqq \sup_{\rho \in \Pi}|\int_{\bS\times\bA} g(s,a)\rmd\rho(s,a)|.
\end{equation*}
We call this semi-norm {\boldmath\emph{$\Pi$-norm}}.  

\item 
Given a Banach space, a positive constant $\epsilon > 0$ and a probability distribution $\nu \in \mathcal{P}(\bS\times\bA)$, we define $\mathcal{B}_{\epsilon,\nu}$, a {\boldmath\emph{$\nu$-perturbation space with scale $\epsilon$}}, as follows:
\begin{equation*}
    \mathcal{B}_{\epsilon,\nu} \coloneqq \{g \in \mathcal{B}\colon \|g\|_\mathcal{B} \le 1, \|g\|_{L^2(\nu)} \le \epsilon\}.
\end{equation*}

\item
The \emph{perturbation response by distribution mismatch} is defined as the radius of $\mathcal{B}_{\epsilon,\nu}$ under $\Pi$-norm,
\begin{equation*}
    \mathcal{R}(\Pi,\mathcal{B},\epsilon,\nu) \coloneqq \sup_{g \in \mathcal{B}_{\epsilon,\nu}} \|g\|_\Pi.
\end{equation*}
\end{enumerate}
\end{defn}

We consider an RL problem whose underlying MDP belongs to a family of MDPs 
\begin{equation*}
    \mathcal{M}= \{M_\theta = (\bS, \bA, P, r_\theta, H, \mu)\colon \theta \in \Theta\}
\end{equation*}
where $\bS$,  $\bA$, $P$, $H$ and $\mu$ are common state space, action space, transition probability\footnote{In \cite{long2022perturbational}, the general case where the transition probability is unknown is also considered and fitted Q-iteration algorithm (Algorithm \ref{alg_fqi}) with kernel function approximation is studied in this case. Here for brevity we only discuss the case where the transition probability is known.}, length of each episode, and initial distribution. The unknown reward function lies in the unit ball of a Banach space $\mathcal{B}$:
\begin{equation}
    \{r_\theta, \theta \in \Theta\} = \mathcal{B}^1.
\end{equation}
In the generative model setting, \cite{long2022perturbational} prove that any RL algorithm $J_n:\theta \rightarrow \bR$ on $\mathcal{M}$ with at most $n$ accesses to the generative model satisfies that
\begin{equation}
    \sup_{\theta \in \Theta} \EE|J_n(\theta) - J^*(M_\theta)| \ge \frac{1}{12}\Delta_\mathcal{M}(n^{-\frac{1}{2}}),
\end{equation}
where 
\begin{equation}
    \Delta_\mathcal{M}(\epsilon) = \inf_{\nu \in \mathcal{P}(\bS\times\bA)} \mathcal{R}(\Pi(P,\mu),\mathcal{B},\epsilon,\nu). 
\end{equation}
Therefore, the perturbational complexity by distribution mismatch gives a lower bound for RL problems on $\mathcal{M}$. Note that in \cite{long2022perturbational}, it is assumed that we can only obtain a noisy reward with a standard normal noise in the generative model, rather than the exact reward. On the other hand, if $\mathcal{B}$ is the Barron space or an RKHS, then the output $\hat{\pi}_\theta$ of Algorithm \ref{alg:FittedReward} satisfies:
\begin{equation}
\sup_{\theta\in\Theta}|J(M_\theta,\hat{\pi}_\theta) - J^*(M_\theta)| \le \tilde{O}(H\Delta_\mathcal{M}(n^{-\frac{1}{4}})).
\end{equation}
Therefore, the perturbational complexity by distribution mismatch also gives an upper bound for RL problems on $\mathcal{M}$.
\begin{algorithm}[ht]
\caption{Fitted reward algorithm}
{
\KwIn{MDP family $\mathcal{M}$, generative model of MDP $(\bS,\bA,H, P,r_\theta,\mu)$, sampling distribution $\hat{\nu} = \argmin_{\nu \in \mathcal{P}(\bS\times\bA)}\mathcal{R}(\Pi(P_0,\mu),\mathcal{H}_k,n^{-\frac{1}{4}},\nu).$}
}

\For{$h = 1,2,\dots,H$}{
Sample $(s_1,a_1),\dots,(s_n,a_n)$ i.i.d. from $\hat{\nu}$ \\
Sample $r_h^1,\dots,r_h^n$ from $\mathcal{N}(r_\theta(h,s_1,a_1),1),\dots,\mathcal{N}(r_\theta(h,s_n,a_n),1)$, respectively\\
Compute $\hat{r}_\theta(h,\cdot)$ as the minimizer of the optimization problem
\begin{equation}
\label{optimization_problem1}
    \min_{\|r\|_\mathcal{B} \le 1} \sum_{i=1}^{n}[r(s_i,a_i) - r_h^i]^2
\end{equation}
}
Collect the fitted reward function to form the MDP $(\bS,\bA,H, P,\hat{r}_\theta,\mu)$, of which both reward function and transition are known. Denote it as $\hat{M}_\theta$.

\KwOut{
$\hat{\pi}_\theta$ as the optimal policy of $\hat{M}_\theta$.
}
\label{alg:FittedReward}
\end{algorithm}

The perturbational complexity by distribution mismatch can also be used to construct various RL problems that suffer from the curse of dimensionality \cite{long2022perturbational}. The first example involves a state space $\bS$ consisting of a single point $s_0$, while the action space $\bA$ is $\mathbb{S}^{d-1}$ and $H = 1$. In this setting, the RL problem essentially aims to find the maximum value of the reward function lying in the unit ball of $\mathcal{B}$ based on the values of $n$ points. We can prove that when $\mathcal{B}$ is the Barron space and the RKHS corresponding to the Laplacian kernel and NTK, the convergence rate can be bounded below by the eigenvalue decay. Therefore, if we consider the RKHS corresponding to the Laplacian kernel or neural tangent kernel, the convergence rate suffers from the curse of dimensionality. We can then conclude that if we want to solve RL problems with high dimensional action space, we need to assume the decay of the eigenvalue is fast enough to break the curse of dimensionality. The other example involves a high-dimensional state space and finite action space. For any dimension $d \ge 2$, length of each episode $H \in \mathbb{N}^+$ and positive constant $\delta > 0$, we define an MDP family $\mathcal{M}_{d,H,\delta}$ as follows:
\begin{align*}
    &\mathcal{S} = \mathbb{S}^{d-1}, \quad \mathcal{A} = \{0,1\}, \quad H = H,  \quad \mu = \text{Uniform}_{\mathbb{S}^{d-1}},\\
    &\{r_{\theta_r}:\theta_r \in \Theta_r\} = \{r: \|r(h,\cdot)\|_{\mathcal{H}_k} \le 1, \forall h \in [H]\}, \\ &k((s,a),(s',a')) = \exp(-\|s-s'\|),  \quad 
    P(\,\cdot\,|\,h,s,a) = \delta_{T_{a,h} s}(\cdot), \\
    &T_{a,h} s = \begin{cases}
       (\phi_1,\dots,\phi_{h_d} + \delta,\dots,\phi_d), \text{ when } a = 0,\\
        (\phi_1,\dots,\phi_{h_d} - \delta,\dots,\phi_d), \text{ when } a = 1,
    \end{cases} 
\end{align*}
where $h_d = h\mod d$ and we use the spherical coordinates $(\phi_1,\dots,\phi_d)$ to denote the points on $\mathbb{S}^{d-1}$. Then we can show that there exist no universal constants $\alpha, \beta > 0$ and constant $C_d > 0$ only depending on $d$ such that
\begin{equation*}
   \sup_{\delta > 0}\Delta_{\mathcal{M}_{d,H,\delta}}(n^{-\frac{1}{2}}) \le C_d H^\alpha (\frac{1}{n})^\beta
\end{equation*}
holds for all $n, H \in \mathbb{N}^{+}$ and $d \ge 2$. Therefore, the above RL problems cannot be solved without the curse of dimensionality.

\section{Discussion and Conclusion}
\label{sec:conclusion}
In this paper, we review existing research on reinforcement learning with function approximation. The results in the tabular and linear settings are well-developed because methods such as $L^\infty$ and UCB estimation can be used to handle the phenomenon of distribution mismatch. When a generative model is available, the perturbational complexity by distribution mismatch can be used to measure the impact of distribution mismatch and assess the difficulty of reinforcement learning problems in the nonlinear setting. However, it remains unclear how to extend these results to the episodic setting, and it is still an open question how to use perturbational complexity information to guide the design of efficient reinforcement learning algorithms in practice.

Approximation error is also an important topic in RL, especially in the nonlinear setting. Apart from the Theorem \ref{thm:lin_mdp} for linear space, Theorem \ref{thm:rkhs} for RKHS, and the condition \eqref{barron_approx} for Barron space, there are limited results in this area, particularly for deterministic MDPs. We remark that the solution of the continuous-time Hamilton-Jacobi-Bellman equation, which is the value function of continuous-time MDPs, can be approximated by neural networks, see, e.g., \cite{hutzenthaler2022overcoming}. However, it is not clear whether this result can be applied to discrete-time MDPs.
Computational issues are another important topic in reinforcement learning, particularly for reinforcement learning with neural function approximation. The convergence of the gradient descent method of neural networks in the mean field regime is still not well-understood. We hope that further research will be conducted on these topics.

Finally, a significant gap exists between the current theory and practice of reinforcement learning, even in the absence of function approximation. The majority of theoretical results focus on algorithms that employ strategic exploration, such as the UCB method \cite{azar2017minimax,jin2018q,jin2020provably,yang2020function}. However, RL algorithms in practice often utilize the random exploration. Theoretical research suggests that, in the worst-case scenario, RL with random exploration exhibits exponential difficulty with respect to the horizon \cite{dann2022guarantees}, which does not accurately explain practical performance. While some theoretical studies \cite{liu2018simple,laidlaw2023bridging} have examined instance-based bounds by identifying specific RL problem properties that lead to better performance than the worst case when random exploration is employed, these properties do not fully account for the success of all practical RL problems, nor do they address function approximation. Furthermore, many practical techniques, such as reward shaping, experience replay, and pre-trained policies, have not been sufficiently explored in theoretical research to explain their positive impact on RL algorithm performance. It is essential for future research to bridge the gap between theory and practice, particularly in the presence of function approximation.

\bibliography{ref}
\bibliographystyle{plain}
\end{document}